%% file: main.tex
\documentclass[preprint,12pt]{elsarticle}

\usepackage{amssymb}
\usepackage{amsmath,amsfonts}

\usepackage{algorithmic}
\usepackage{algorithm}

\usepackage{hyperref}

\usepackage{svg}
\usepackage{booktabs}
\usepackage{adjustbox}
\usepackage{wrapfig}
\usepackage{subcaption}

\usepackage{amsthm}
\newtheorem{proposition}{Proposition}
\newtheorem{assumption}{Assumption}

\definecolor{commentcolor}{RGB}{110,154,155}   
\newcommand{\PyComment}[1]{\ttfamily\textcolor{commentcolor}{\# #1}}  
\newcommand{\PyCode}[1]{\ttfamily\textcolor{black}{#1}} 

\journal{Neural Networks}

\begin{document}

\begin{frontmatter}


\title{Feature propagation as self-supervision signals on graphs}

\author[upc]{Oscar Pina \corref{corresponding}}
\ead{oscar.pina@upc.edu}
\author[upc]{Verónica Vilaplana}
\ead{veronica.vilaplana@upc.edu}

\cortext[corresponding]{Corresponding author.}
\affiliation[upc]{organization={Universitat Politècnica de Catalunya - BarcelonaTech (UPC)},
            city={Barcelona},
            country={Spain}}

\input{sections/01_abstract}

\end{frontmatter}

\input{sections/02_introduction}

\input{sections/03_background}
\input{sections/04_method}
\input{sections/05_eval}
\input{sections/06_ablations}

\input{sections/07_discussion}
\input{sections/08_conclusions}

\section{Acknowledgements}
This work has been supported by the Spanish Research Agency (AEI) under project PID2020-116907RB-I00 of the call MCIN/ AEI /10.13039/501100011033 and the FI-AGAUR grant funded by Direcció General de Recerca (DGR) of Departament de Recerca i Universitats (REU) of the Generalitat de Catalunya.

\bibliography{refs}
\bibliographystyle{elsarticle-num}

\appendix
\input{sections/appendix}

\end{document}

%% file: sections/01_abstract.tex
\begin{abstract}
    Self-supervised learning is gaining considerable attention as a solution to avoid the requirement of extensive annotations in representation learning on graphs. Current algorithms are based on contrastive learning, which is computation an memory expensive, and the assumption of invariance under certain graph augmentations. However, graph transformations such as edge sampling may modify the semantics of the data so that the iinvariance assumption may be incorrect. We introduce Regularized Graph Infomax (RGI), a simple yet effective framework for node level self-supervised learning that trains a graph neural network encoder by maximizing the mutual information between output node embeddings and their propagation through the graph, which encode the nodes' local and global context, respectively. RGI do not use graph data augmentations but instead generates self-supervision signals with feature propagation, is non-contrastive and does not depend on a two branch architecture. We run RGI on both transductive and inductive settings with popular graph benchmarks and show that it can achieve state-of-the-art performance regardless of its simplicity. \footnote{The code is available at \href{https://github.com/oscar97pina/gssl-rgi}{https://github.com/oscar97pina/gssl-rgi}}
\end{abstract}



\begin{keyword}



Graph neural network \sep self-supervised learning \sep graph representation learning \sep regularization

\end{keyword}

%% file: sections/02_introduction.tex
\section{Introduction}
\label{sec:introduction}

The primary goal of self-supervised learning is to learn meaningful representations from large amounts of unlabeled data that can be employed to efficiently fit other downstream tasks with a small amount of labeled samples. The challenge is to define an appropriate auxiliary task that leads the model to capture the relevant parts of the input data.

A popular solution within the computer vision field is to force the encoder to be invariant under certain transformations applied to the input, such as rotation and cropping, by generating two views of the same image by applying these random augmentations and maximizing the agreement between their representations. Owing to their success, these works have been adapted by the graph learning community in order to train graph neural networks encoders in a self-supervised manner \cite{chen2020big, chen2020simple, grill2020bootstrap, zbontar2021barlow}. Consequently, graph specific data augmentation techniques are required to generate the graph views. Graph transformations can focus on both node attributes and the topology of the graph. Indeed, popular choices are node attribute masking and edge sampling \cite{thakoor2021bootstrapped, zhang2021canonical, bielak2021graph, Zhu:2020vf}. However, the idea behind invariance via data augmentations roots in the fact that it is assumed that these transformations do not change the semantics of the data and the information contained about the downstream task is kept. Whereas we can understand image augmentations, it is not clear how graph transformations modify its semantics nor if they can be applied to all graph domains \cite{NEURIPS2021_ff1e68e7}. To overcome this issue, graph diffusion has been proposed to create the alternative views \cite{icml2020_1971}, and other works attempt to get rid of the augmentations by designing strategies that leverage the local neighborhood of the nodes  \cite{peng2020graph, lee2021augmentation, https://doi.org/10.48550/arxiv.2204.04874}.

In this manuscript, we introduce \textit{Regularized Graph Infomax (RGI)}, a simple yet effective self-supervised learning framework for graph structured data. RGI trains the encoder based on \textit{embedding propagation} as self supervision signals, which consists of propagating the output embeddings through the graph and using them as prediction target from the node embeddings, which leads to context-aware embeddings that do not encode unknown invariances. Additionally, variance-covariance regularization \cite{bardes2022vicreg} is used in order to avoid the collapse of the representations, which is generally more efficient than contrastive methods. Therefore, the algorithm is augmentation-free, non-contrastive, and does not require a two-branch architecture nor complex training strategies.

This document is organized as follows, in Section \ref{sec:background} we provide the necessary knowledge to fully understand the arguments and statements of the method, in Section \ref{sec:method} we detail our algorithm RGI, motivate and explain the usage of feature propagation as self-supervision signals (Section \ref{sec:method_propagation}), provide a \textit{local-global} context perspective of RGI and feature propagation (Section \ref{sec:method_locglob}), detail the loss function (Section \ref{sec:method_loss}) and show that RGI is maximizing the mutual information between the node embeddings and their propagation through the graph (Section \ref{sec:method_motiv}). Then, we perform empirical evaluation on popular graph benchmarks in Section \ref{sec:eval} and analyse the influence of different parts of the algorithm in the ablation study of Section \ref{sec:ablations}. Finally, we discuss RGI and the results in Section \ref{sec:discussion} and conclude the work.

%% file: sections/03_background.tex
\section{Background and related work}
\label{sec:background}

\subsection{Information theory for representation learning}
\paragraph{Mutual information} The mutual information (MI) between two random variables $X$ and $Y$, $I(X;Y)$ is a symmetric quantity, $I(X;Y)=I(Y;X)$, that measures how much information one variable carries about the other. Formally, it is defined as:
\begin{equation}
I(X;Y) = H(X) - H(X|Y) = H(Y) - H(Y|X)
\label{eq:mi}
\end{equation}
where $H(X)$ is the entropy of $X$ and $H(X|Y)$ is the conditional entropy of $X$ given $Y$. This quantity can be lower-bounded by the expected reconstruction error, which is usually employed in generative models:
\begin{equation}
I(X;Y) = H(X) - H(X|Y) \geq H(X) - \mathcal{R}(X|Y)
\label{eq:mi_bound}
\end{equation}
where $\mathcal{R}(X|Y)$ is the expected reconstruction error of $X$ given $Y$. In practice, this expected error is approximated with the square loss or the cross-entropy loss.

\paragraph{InfoMax principle} The InfoMax principle \cite{36} states that a neural network can be trained in a self-supervised manner by maximizing the mutual information between its input $X$ and its output $Y$, $I(X;Y)$. To do so, \cite{Bell1995-je}  suggests that it is enough to maximize the entropy of $Y$, which, for the example of a one layer n-to-n network, is achieved by maximizing the logarithm of the jacobian of the weights. The intuition is that this quantity can be seen as the log of the volume space of $Y$ onto which the values of $X$ are mapped.

\subsection{Multi-view representation learning} 
The InfoMax principle is extended to a multi-view approach, in which rather than maximizing the MI between the input and output of the network, the agreement is maximized between the representations of two different views of the input. This scenario has two main advantages, (i) the loss is computed in the representation space, which is in general lower dimensional and avoids focusing on small details of the input and (ii) views can be defined to capture different aspects of the data \cite{Tschannen2020On}.

\paragraph{Local - global MI} Deep InfoMax \cite{hjelm2018learning} trains an encoder maximizing the average mutual information between local patches and global representations of an image. Deep Graph InfoMax \cite{velickovic2018deep} and InfoGraph \cite{sun2019infograph} extend this work to the graph domain, targetting the MI between node and graph level embeddings. The graph representation is obtained with a global pooling layer applied to the local node embeddings. In DGI, since most datasets consist of one single graph, the authors create a corrupted version of the graph by shuffling the node features and contrasting negative and positive pairs.

\paragraph{Invariance via data augmentation} SimCLR \cite{chen2020simple}, BYOL \cite{grill2020bootstrap}, Barlow Twins \cite{zbontar2021barlow} and VICReg \cite{bardes2022vicreg}, among others, create two augmented views of an image via data augmentation, such as image rotation and cropping, and train the encoder to be invariant to those augmentations, not necessarily with MI objectives. These works are also extended for self-supervised graph representation learning. For instance, GRACE \cite{Zhu:2020vf} follows a similar approach to SimCLR based on contrastive learning, BGRL \cite{thakoor2021bootstrapped} employs the same asymmetric scheme than BYOL and G-BT, \cite{bielak2021graph} directly extends Barlow Twins cross-covariance regularization objective.

\subsection{Avoiding collapse}
Maximizing the agreement between views can lead to a total collapse in which the encoder outputs the same representation independently from the input. In order to diminish this phenomena, either architectural, regularization or training tricks can be employed:

\paragraph{Contrastive learning} Contrastive methods not only encourage the two views of the input (positive pairs) to be similar, but they also force views from different inputs (negative pairs) to be different \cite{chen2020simple, Zhu:2020vf, NEURIPS2021_ff1e68e7, icml2020_1971, velickovic2018deep, hjelm2018learning}. They have been successfully applied to all data domains, nonetheless, their performance is highly influenced by the number of negative pairs, usually requiring many of them to work efficiently.

\paragraph{Knowledge distillation} Knowledge distillation methods do not need negative pairs \cite{grill2020bootstrap, thakoor2021bootstrapped}. Instead, they construct a teacher-student asymmetric architecture combined with a stop-gradient operation. Concretely, the networks are fed with two augmented views and the student network is trained to predict the output of the teacher model. Collapse is avoided by not backpropagating gradients through the teacher network to update its parameters, but setting them to be a moving average of the student's parameters.

\paragraph{Covariance regularization} Covariance regularization consists of extending the loss function by including regularization terms on the covariance matrix of the representations, forcing high variances for every feature and low co-variances \cite{bardes2022vicreg, zbontar2021barlow, zhang2021canonical}. Under the gaussianity assumption, these loss terms attempt to regularize the entropy of the latent space to avoid the collapse \cite{https://doi.org/10.48550/arxiv.2207.10081}.

\begin{figure}[ht]
\caption{Visual illustration of RGI}
\label{fig:rgi}
\vskip 0.2in
\begin{center}
\centerline{\includegraphics[width=0.7\textwidth]{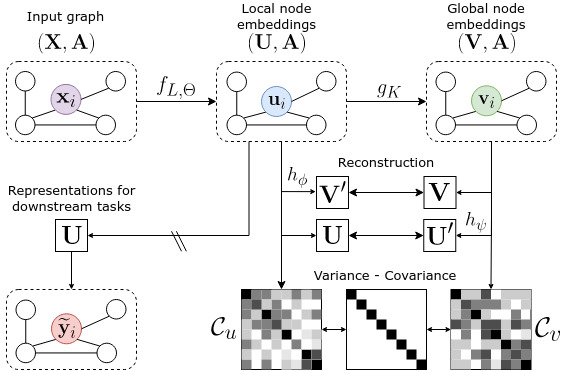}}

\vskip 0.2in
\footnotesize{Given the input graph $\mathcal{G}=(\mathbf{X},\mathbf{A})$, a GNN encoder $f_{\Theta}$ extracts node features $\mathbf{U} = f_{\Theta}(\mathbf{X},\mathbf{A})$ and they are propagated through the graph during $K$ steps to obtain the nodes' global context $\mathbf{V}=g_K( \mathbf{U}, \mathbf{A})$. RGI maximizes the MI between $\mathbf{U}$ and $\mathbf{V}$ based on the lower bound of the MI. To do so, two auxiliary neural networks $h_{\phi}$ and $h_{\psi}$ are trained to predict $V$ from $U$ and $U$ from $V$, respectively. Additionally, the covariance matrices of the local and global node embeddings are regularized to have large diagonal elements and off-diagonal elements close to zero, attempting to regularize the entropy of $\mathbf{U}$ and $\mathbf{V}$. The neural networks $h_{\phi}$ and $h_{\psi}$, as well as $\mathbf{V}$ are ignored during inference.}

\end{center}
\vskip -0.2in
\end{figure}

%% file: sections/04_method.tex
\section{Regularized graph infomax}
\label{sec:method}

In this section, we introduce the concept of feature propagation as self-supervision signals and detail our algorithm \textit{RGI - Regularized Graph Infomax} for self-supervised learning on graphs, which we show that maximizes the mutual information between the representations output by a graph neural network encoder and their propagation through the graph. The algorithm is described in Algorithm \ref{algo:rgi} and a visual illustration is shown in Figure \ref{fig:rgi}.

\begin{algorithm}[tb]
\caption{RGI}
\label{algo:rgi}
\begin{algorithmic}
   \STATE {\bfseries Input:} node features $\mathbf{X}$; adjacency matrix $\mathbf{A}$; parameters $\Theta$, $\phi$ and $\psi$; backbone $f_{\Theta}$; reconstruction networks $h_{\phi}$ and $h_{\psi}$; feature propagation function $g_K$.
   \REPEAT
   \STATE // obtain node embeddings
   \STATE $\mathbf{U} = f_{\Theta}({\mathbf{X}}, {\mathbf{A}})$
   \STATE // propagate node embeddings
   \STATE $\mathbf{V} = g_K({\mathbf{U}, \mathbf{A}})$
   \STATE // reconstruction
   \STATE $\mathbf{V}'=h_{\phi}(\mathbf{U})$
   \STATE $\mathbf{U}'=h_{\psi}(\mathbf{V})$
   \STATE // covariance matrices
   \STATE $\mathcal{C}_u= \frac{1}{N} \bar{\mathbf{U}}^T \bar{\mathbf{U}} $
   \STATE $\mathcal{C}_v     = \frac{1}{N} \bar{\mathbf{V}}^T \bar{\mathbf{V}} $
   \STATE // loss function
   \STATE $\mathcal{L}_{rec} =  \parallel \mathbf{U} - \mathbf{U}' \parallel^2_F +  \parallel \mathbf{V} - \mathbf{V}' \parallel^2_F $
   \STATE $\mathcal{L}_{var} = (1-$diag$(\mathcal{C}_u))^2 + (1-$diag$(\mathcal{C}_v))^2$
   \STATE $\mathcal{L}_{cov} = ($off-diag$(\mathcal{C}_u))^2 + ($off-diag$(\mathcal{C}_v)^2$
   \STATE $\mathcal{L} = \lambda_1 \mathcal{L}_{rec} + \lambda_2 \mathcal{L}_{var} + \lambda_3 \mathcal{L}_{cov} $
   \STATE // update parameters
   \STATE $\Theta, \phi, \psi \leftarrow \bigtriangledown_{\Theta, \phi, \psi} \mathcal{L}$
   \UNTIL{convergence}
   \STATE {\bfseries return} $\mathbf{U}$
\end{algorithmic}
\end{algorithm}

\subsection{Context and notation}
\label{sec:method_context}
Let $\mathcal{G} = (\mathcal{V}, \mathcal{E})$ be a graph of $N$ nodes where $\mathcal{V}$ is the node set and $\mathcal{E}$ the edge set. Node attributes $X$ come from a $d$-dimensional distribution and the nodes' realizations are arranged in a matrix $\mathbf{X} = \{ \mathbf{x}_i \}_{i=1}^{N} \in \mathbb{R}^{N \times d}$, where $\mathbf{x}_i \in \mathbb{R}^{d}$ are the attributes of node $i$. $\mathcal{E}$ comes in the form of (unweighted) adjacency matrix $\mathbf{A} \in \{0,1\}^{N \times N}$, $\mathbf{A}_{ij} = 1$ if the edge $e_{ij} \in \mathcal{E}$, $0$ otherwise. The $K$-hop neighborhood of a node $i$, $\mathcal{G}_i^{(K)}$, is represented by the set of neighbors' attributes, $\mathbf{X}_i^{(K)}$ and the induced adjacency matrix $\mathbf{A}_i^{(K)}$. The same notation criterion of $X$, $\mathbf{X}$, $\mathbf{x}_i$ and $\mathbf{X}_i$ will be employed for other node hidden representations.

The goal of self-supervised learning on graphs is to fit a graph neural network encoder $f_{\Theta} : \mathbb{R}^{N \times d} \times \{0,1\}^{N \times N} \xrightarrow{} \mathbb{R}^{N \times D}$, $\mathbf{U} = f_{\Theta}(\mathcal{G}) = f_{\Theta}(\mathbf{X}, \mathbf{A})$, parametrized by $\Theta$, that obtains a $D$-dimensional vector representation $\mathbf{u}_i$ for every node of the input graph without relying on node annotations.

\subsection{Feature propagation as supervision signals}
\label{sec:method_propagation}

RGI train an encoder to maximize the mutual information between the embeddings output by the encoder, $\mathbf{U} = f_{\Theta}(\mathbf{X}, \mathbf{A})$, and their propagation through the graph, $\mathbf{V} = g_K(\mathbf{U}, \mathbf{A})$, named \textit{local} and \textit{global} node views or embeddings, respectively. 

As shown in Section \ref{sec:method_locglob}, the propagated embeddings capture the global information and structure of the graph while being particular to every node. Leveraging them as supervision signals, the model is trained to output more informative node embeddings that are aware of both their local and global context within the graph. The propagation function $g_K$ is implemented as a $K$ order polynomial of the graph adjacency matrix or any of its variants, such as the symmetrically normalized adjacency matrix, so that RGI does not rely on a two branch architecture and the computational burden to obtain the supervision signals is minimal compared to other methods.

We perform feature propagation over the embeddings $\mathbf{U}$ to create the supervision signals rather than the propagation of the raw input features $\mathbf{X}$ since real world graph data is usually incomplete and sparse. Additionally, there may be scenarios without initial node features.

\subsection{Local and global contexts}
\label{sec:method_locglob}
We name $\mathbf{U}$ and $\mathbf{V}$ as node local and global node embeddings, respectively. Note, however, that the global definition is different from the global pooling proposed in DGI \cite{velickovic2018deep}, since we define node level global embeddings rather than graph level. The reason to call $\mathbf{V}$ \textit{global} is as follows. For a target node $i$, we obtain its local embedding with a $L$-layer GNN, $\mathbf{u}_i=f_{\Theta}(\mathbf{X}_i^{(L)}, \mathbf{A}_i^{(L)})$, so that it depends on its $L$-hop neighborhood. Afterwards, its global context is computed by propagating the representations through the graph during $K$ steps, $\mathbf{v}_i=g_K(\mathbf{U}_i^{(K)}, \mathbf{A}_i^{(K)})$. Therefore, $\mathbf{v}_i$ contains information of the $(L+K)$-hop neighborhood of $i$. It is known that for a small-world network $\mathcal{G}$ of $N$ nodes, the diameter of the graph is $diam(\mathcal{G})=log(N)$. Consequently, in a small world network, $\mathbf{v}_i$ can encode global information of every node in the graph as long as $L + K \simeq log(N)$.

\subsection{Loss function}
\label{sec:method_loss}
Based on Equation \ref{eq:mi_bound}, the mean squared error is employed as loss function to maximize the mutual information. RGI includes an auxiliary network $h_{\phi}$ implemented as a multi-layer perceptron to predict the local embeddings from the global embeddings. The network is jointly trained with the encoder, but discarted afterwards.

\begin{equation}
    \label{eq:loss_rec}
    \mathcal{L}_{rec}(U|V) = \frac{1}{N} \sum_{i=1}^{N} \parallel \mathbf{u}_i - h_{\phi}( \mathbf{v}_i ) \parallel_2^2
\end{equation}

However, as the representations $U$ are not fixed, but they depend on the encoder parameters instead, only addressing the reconstruction error would lead the encoder to a collapse, in which the input is ignored and it outputs a constant representation. To alleviate this problem, we incorporate covariance matrix regularization loss terms to avoid the collapse of the representations \cite{zbontar2021barlow, bardes2022vicreg, bielak2021graph, zhang2021canonical}:

\begin{equation}
    \label{eq:loss_var}
    \mathcal{L}_{var}(U) = \frac{1}{D} \sum_{n=1}^{D} \left ( 1 - \mathcal{C}_{nn} \right )^2
\end{equation}

\begin{equation}
    \label{eq:loss_cov}
    \mathcal{L}_{cov}(U) = \frac{1}{D} \sum_{n=1}^{D} \sum_{m \neq n} \mathcal{C}_{nm}^2
\end{equation}

where $\mathcal{C}$ is the sample covariance matrix of $U$. Concretely, the former guides the diagonal elements (variances) to be close to one whereas the latter forces the non-diagonal elements (covariances) to be zero. Intuitively, maximizing the variance avoids the total collapse to a constant representation. On the other hand, covariance minimization encourages the encoder to leverage the whole capacity of the representation space rather than projecting the points to a lower dimensional subspace, also known as dimensional collapse \cite{hua2021feature}. The loss function is a weighted combination of these three terms:

 \begin{equation}
     \label{eq:loss_u}
     \mathcal{L}_u = \lambda_1 \mathcal{L}_{rec}(U|V) + \lambda_2 \mathcal{L}_{var}(U) + \lambda_3 \mathcal{L}_{cov}(U)
 \end{equation}

Where $\lambda_1, \lambda_2, \lambda_3 \in \mathbb{R}$ are weight hyperparameters. In practice, we symmetrize this loss by also predicting $\mathbf{v}_i$ from $\mathbf{u}_i$ with another auxiliary network $h_{\psi}$ and applying variance-covariance regularization to $V$.

\subsection{Mutual information maximization}
\label{sec:method_motiv}

In this section, we show that, as previously stated in Section \ref{sec:method_propagation}, RGI maximizes the mutual information between the node embeddings and their propagation.

\begin{assumption}
    \label{ass:gau_u}
    The node embeddings $\mathbf{U}$ follow a Gaussian distribution $U \sim N(\mu_u, \Sigma_u)$.
\end{assumption}

\begin{assumption}
    \label{ass:gau_v}
    The propagated embeddings $\mathbf{V}$ follow a Gaussian distribution $V \sim N(\mu_v, \Sigma_v)$.
\end{assumption}

\begin{proposition}
    \label{prop:mi_max}
    Based on Assumption \ref{ass:gau_u}, minimizing the objective $\mathcal{L}_u$ of Equation \ref{eq:loss_u} maximizes the mutual information between $U$ and $V$. 
\end{proposition}
\begin{proof}
    The proof is based on the MI lower bound of Equation \ref{eq:mi_bound}. In order to maximize the MI between $U$ and $V$, it is necessary to address both the maximization of $H(U)$ and the minimization of $\mathcal{R}(U|V)$, as the representations $U$ are not fixed and only optimizing $\mathcal{R}(U|V)$ would also affect the value of $H(U)$. The first term of the loss in Equation \ref{eq:loss_u} is the reconstruction error $\mathcal{R}(U|V)$, since our representations are continuous-valued, this is achieved with the mean-squared error. Secondly, under the Gaussianity assumption, the entropy of $U$ is proportional to the logarithm of the determinant of the covariance matrix, that is,  $H(U) \propto log(\left | \Sigma_u \right |)$. Being the logarithm an increasing function, the entropy maximization can be tackled by maximizing $\left | \Sigma_u \right |$. Although we do not have access to the true covariance matrix, it is approximated by targetting the sample covariance $\mathcal{C}$. Rather than directly maximizing the determinant, the proxy consists of maximizing the diagonal elements of the matrix while also forcing the non-diagonal elements to be close to 0.
    
    Despite the fact that the mutual information is a symmetric quantity, the approximation based on the reconstruction error and sample covariance matrix are not. Hence, the loss is symmetrized given the Assumption \ref{ass:gau_v}.
\end{proof}





%% file: sections/05_eval.tex
\section{Evaluation}
\label{sec:eval}

In this section, we evaluate the quality of the node level representations output by our method on both transductive and inductive settings. 

\paragraph{Datasets} For transductive learning we run RGI on 4 popular benchmarks: \textit{Amazon Computers}, \textit{Amazon Photos}, \textit{Coauthor CS} and \textit{Coauthor Physics}. Moreover, we include the large scale \textit{ogbn-arxiv} dataset. Finally, inductive learning evaluation is addressed with the challenging \textit{PPI} dataset. The statistics of the datasets are shown in Table \ref{tab:datasets}.

\paragraph{Linear evaluation} We follow the linear evaluation protocol on graphs to assess the quality of the representations as proposed in \cite{velickovic2018deep}. It consists of first fitting a GNN encoder in a fully self-supervised manner, freezing the weights of the encoder, obtaining the node-level representations and fitting a linear classifier to a downstream task without backpropagating the gradients through the encoder. For comparability with other methods, the transductive settings are evaluated in terms of accuracy of the predictions whereas for the \textit{PPI} dataset we employ the micro-average F1-score. As usual, since there is no public split for these datasets, \textit{Amazon Computers}, \textit{Amazon Photos}, \textit{Coauthor CS} and \textit{Coauthor Physics} are randomly split into train/validation/test (10\% / 10\% / 80\%). The ogbn-arxiv dataset is split with the partition provided by Open Graph Benchmark \cite{hu2021open}. To evaluate on the \textit{PPI} dataset, we employ the standard pre-defined split, which has 20 graphs to fit the model, 2 graphs for validation and another two for testing.

\subsection{Transductive learning}

\paragraph{Architecture} As in \cite{velickovic2018deep, thakoor2021bootstrapped}, among others, we fix the encoder $f_{\Theta}$ to be a $L=2$ layer GCN \cite{kipf2017semisupervised} for the Amazon and Co-authorship datasets. We have set the output dimensionality to be $512$ for all of them and the hidden, $1024$. We include batch normalization \cite{10.5555/3045118.3045167} and ReLU activation after the first layer but none of them is included after the second convolutional layer.  We also apply dropout regularization to the input graph. The architecture for the \textit{ogbn-arxiv} dataset is slightly different. Based on \cite{thakoor2021bootstrapped}, we employ a $L=3$ GCN \cite{kipf2017semisupervised} encoder, with layer normalization and ReLU activation after the first and the second layer.

To obtain the global node embeddings, we propagate the embeddings with the normalized adjacency matrix without self-loops for $K=1$ step, that is:

\begin{equation}
    \label{eq:gk_transductive}
    \mathbf{V} = g_K(\mathbf{U},\mathbf{A})= \left ( {\mathbf{D}}^{-\frac{1}{2}} {\mathbf{A}} {\mathbf{D}}^{-\frac{1}{2}} \right ) \mathbf{U}
\end{equation}
Finally, the networks $h_{\phi}$ and $h_{\psi}$ are implemented with a two layer FCNN and ReLU activation in the hidden layer. However, no batch normalization is employed for these auxiliary networks.

\paragraph{Numerical results} Table \ref{tab:results_transductive} shows the mean accuracy of the linear evaluation protocol on the transductive graph settings. Our results are the average of 20 model weight initializations and data splits. The other results are extracted from previous reports. We can observe that RGI performs competitively in all datasets despite its simplicity and achieves state of the art in some of them. 
In except of \textit{Amazon Computers} dataset, RGI is trained for $1,000$ epochs whereas other methods such as BGRL require $10,000$ epochs \cite{thakoor2021bootstrapped}.  For instance, when BGRL is trained only for $1,000$ epochs, its performance in all datasets is dropped \cite{bielak2021graph}. Finally, as for the \textit{ogbn-arxiv} dataset, while RGI performs competitively, state-of-the-art performance is achieved by masking and augmentation-based methods, which suggest that graph data augmentations and the invariance to them are an appropriate solution for this particular dataset and task.

\begin{table}[ht]
\centering
\caption{Transductive and inductive datasets}
\label{tab:datasets}
\begin{adjustbox}{width=1\textwidth}
\small
\begin{tabular}{lccccc}
\toprule
\textbf{Name} & \textbf{Task} & \textbf{Num. Nodes} & \textbf{Num. Edges} & \textbf{Node features} & \textbf{Num. Classes} \\
\midrule
Amazon Computers & Transductive & 13,752 & 245,861 & 767 & 10\\
Amazon Photos & Transductive & 7,650 & 119,081 & 745 & 8 \\
Coauthor CS & Transductive & 18,333 & 81,894 & 6,805 & 15 \\
Coauthor Physics & Transductive & 34,493 & 247,962 & 8,415 & 5 \\
\midrule
ogbn-arxiv & Transductive & 168,343 & 1,166,243 & 128 & 40 \\
\midrule
PPI (24 graphs) & Inductive & 56,944 & 818,716 & 50 & 121 (multilabel) \\

\bottomrule
\end{tabular}
\end{adjustbox}
\vskip 0.1in
\end{table}

\begin{table}[ht]
\centering
\caption{Classification accuracies on transductive datasets averaged for 20 weight initializations.}
\label{tab:results_transductive}
\begin{adjustbox}{width=1\textwidth}
\small
\begin{tabular}{lccccc}
\toprule
\textbf{Method} & \textbf{Am. Computers} & \textbf{Am. Photos} & \textbf{Co. CS} & \textbf{Co. Physics} & \textbf{Ogbn-Arxiv} \\
\midrule
Raw ft. &  73.81 $\pm$ 0.00 & 78.53 $\pm$ 0.00 & 90.37 $\pm$ 0.00 & 93.58 $\pm$ 0.00 &  55.50 $\pm$ 0.23\\
\midrule
Random-Init & 86.46 $\pm$ 0.38 & 92.08 $\pm$ 0.48 & 91.64 $\pm$ 0.29 & 93.71 $\pm$ 0.29 & 68.94 $\pm$ 0.15 \\
DGI \cite{velickovic2018deep}& 83.95 $\pm$ 0.47 & 91.61 $\pm$ 0.22 & 92.15 $\pm$ 0.63 & 94.51 $\pm$ 0.52 & 70.34 $\pm$ 0.16\\
GRACE \cite{Zhu:2020vf}& 89.53 $\pm$ 0.35 & 92.78 $\pm$ 0.45 & 91.12 $\pm$ 0.20 & OOM & 71.51 $\pm$ 0.11\\
G-BT \cite{bielak2021graph} & 88.14 $\pm$ 0.33 & 92.63 $\pm$ 0.44 & 92.95 $\pm$ 0.17 & 95.07 $\pm$ 0.17 & 70.12 $\pm$ 0.18 \\
CCA-SSG \cite{zhang2021canonical} & 88.74 $\pm$ 0.28 & 93.14 $\pm$ 0.14 & 93.31 $\pm$ 0.22 & 95.38 $\pm$ 0.06 & 71.24 $\pm$ 0.20 \\
BGRL \cite{thakoor2021bootstrapped} & 90.34 $\pm$ 0.19 & \textbf{93.17 $\pm$ 0.30} & 93.31 $\pm$ 0.13 & 95.73 $\pm$ 0.05 & \textbf{71.64 $\pm$ 0.12} \\
\midrule
RGI (\textit{ours}) & \textbf{90.45 $\pm$ 0.08} & 92.94 $\pm$ 0.09 & \textbf{93.37 $\pm$ 0.07} & \textbf{95.91 $\pm$ 0.09} & 70.33 $\pm$ 0.25 \\
\bottomrule
\end{tabular}
\end{adjustbox}
\vskip 0.05in
\small
\footnotesize *OOM refers to out-of-memory error in a 16GB GPU.
\vskip 0.1in
\end{table}

\subsection{Inductive learning}
\label{sec:eval_inductive}

\begin{table}[ht]
\centering
\caption{Classification micro-average F1 score on PPI dataset averaged for 20 weight initializations.}
\label{tab:results_ppi}
\small
\begin{tabular}{lc}
\toprule
\textbf{Method} & \textbf{PPI} \\
\midrule  
Raw ft. & 42.20 \\
\midrule
Rdm-Init & 62.60 $\pm$ 0.20 \\
DGI \cite{velickovic2018deep} & 63.80 $\pm$ 0.20 \\
GMI \cite{peng2020graph} & 65.00 $\pm$ 0.02 \\
GRACE \cite{Zhu:2020vf} & 69.71 $\pm$ 0.17 \\
G-BT \cite{bielak2021graph} & 70.49 $\pm$ 0.19 \\
BGRL \cite{thakoor2021bootstrapped} &
70.49 $\pm$ 0.05 \\
GraphMAE \cite{hou2022graphmae}* & 63.60 $\pm$ 0.29 \\
\midrule
RGI (\textit{ours}) & \textbf{72.16 $\pm$ 0.11}\\
\bottomrule
\end{tabular}
\vskip 0.05in
\footnotesize *Results are replicated with author's official code but reducing the dimensionality to 512, as it is employed in RGI and others, as well as limiting the linear evaluation method to the one employed in RGI for a fairer comparison.
\end{table}

\paragraph{Architecture} Based on previous reports \cite{bielak2021graph, thakoor2021bootstrapped}, we implement the encoder with a $L=3$ layer GAT \cite{veličković2018graph} with ELU activation and skip connections. Both hidden and output dimensionality are set to $D=512$. Although a mean-pooling scheme would be more appropriate for inductive settings \cite{NIPS2017_5dd9db5e}, we employ the same propagation scheme than in transductive setting of Equation \ref{eq:gk_transductive}. In Section \ref{sec:ablations_propagation} we study the effect of $g_K$. As initial node features are sparse in the PPI dataset, we apply dropout regularization to the graph before propagation. The reconstruction networks $h_{\phi}$ and $h_{\psi}$ have the same architecture as in the transductive setting.

\paragraph{Numerical results} Table \ref{tab:results_ppi} shows the micro-average F1-score of the linear evaluation averaged for 20 runs. RGI outperforms the current state of the art on this challenging dataset and only requires $2,000$ epochs to reach this performance.

%% file: sections/06_ablations.tex
\section{Ablation study}
\label{sec:ablations}
In this section, we perform an exhaustive ablation study to evaluate the influence of the different components of RGI. Generally, results show that RGI performs robustly for different design configurations, specifically in transductive settings.

\subsection{Loss function components} 
\label{sec:ablations_loss}
RGI maximizes the mutual information between local and global node embeddings by tackling the reconstruction error between them as well as entropy regularization, which is composed by a variance and a covariance term. The final objective is defined as a weighted sum of reconstruction, variance and covariance terms by $\lambda_1$, $\lambda_2$ and $\lambda_3$, respectively. In order to measure the importance of every term, we have trained RGI by setting their corresponding weight $\lambda$ to zero and observing the difference in performance. The results are shown in Table \ref{tab:ablations_loss}. We observe that variance and covariance terms are required to avoid the collapse of the representations, specifically in the inductive setting, which is more sensitive to these parameters. Additionally, we observe that the reconstruction error is also needed to make the algorithm achieve state-of-the-art performance, validating hence the usage of propagated embeddings as supervision signals.

\begin{table}[ht]
\centering
\caption{Effect of the components of the loss function. Linear evaluation accuracy (micro-average F-Score for PPI dataset) after fitting and encoder with distinct configurations averaged for 5 weight initializations and data splits.}
\label{tab:ablations_loss}
\begin{adjustbox}{width=1\textwidth}
\small
\begin{tabular}{cccccc}
\toprule
$\mathcal{L}$ & \textbf{Am. Computers} & \textbf{Am. Photos} & \textbf{Co. CS} & \textbf{Co. Physics}  & \textbf{PPI} (F-Score) \\
\midrule
$\lambda_1=0$ & 89.78 $\pm$ 0.09 & 92.25 $\pm$ 0.20 & 93.53 $\pm$ 0.06 & 95.96 $\pm$ 0.05 & 68.83 $\pm$ 0.08 \\
$\lambda_2=0$ & 79.53 $\pm$ 0.96 & 89.62 $\pm$ 0.39 & 84.27 $\pm$ 0.39 & 92.57 $\pm$ 0.22 & 0.00 $\pm$ 0.00 \\
$\lambda_3=0$ & 85.70 $\pm$ 0.21 & 92.32 $\pm$ 0.12 & 92.51 $\pm$ 0.09 & 95.48 $\pm$ 0.07 & 49.43 $\pm$ 0.21 \\
$\lambda_1,\lambda_2,\lambda_3 \neq 0 $ & 90.45 $\pm$ 0.07 &  92.88 $\pm$ 0.13 & 93.37 $\pm$ 0.03 & 95.93 $\pm$ 0.04 & 72.25 $\pm$ 0.09 \\
\bottomrule
\end{tabular}
\end{adjustbox}
\vskip -0.1in
\end{table}

\subsubsection{Feature propagation function}
\label{sec:ablations_propagation}
A propagation function $g_K(\mathbf{U},\mathbf{A})$ is used to create the supervision signals from the node embeddings. In this section, we evaluate different shift operators, namely the mean-propagation $\mathbf{D}^{-1}\mathbf{A}$, the normalized adjacency matrix $\mathbf{D}^{-\frac{1}{2}} {\mathbf{A}} {\mathbf{D}}^{-\frac{1}{2}}$ and the normalized Laplacian $\mathbf{I} - \mathbf{D}^{-\frac{1}{2}} {\mathbf{A}} {\mathbf{D}}^{-\frac{1}{2}}$ as well as different values of $K$ (1, 2 and 5). Note that none of them include the self-loops as opposed to most GNN propagation schemes. Table \ref{tab:ablations_propagation} shows that RGI performs robustly in all settings as long as the global propagation scheme captures the low-frequency components of the data, obtained with the adjacency matrix (either mean pooling or normalized), which is motivated by the homophily of the downstream task on the graph.

\begin{table}[ht]
\centering
\caption{Effect of global propagation scheme. Linear evaluation accuracy (micro-average F-Score for PPI dataset) for different schemes to obtain the global node embeddings propagating the local ones averaged for 5 weight initializations and data splits.}
\label{tab:ablations_propagation}
\begin{adjustbox}{width=1\textwidth}
\small
\begin{tabular}{ccccccc}
\toprule
$g(\mathbf{A})$ & $K$ & \textbf{Am. Computers} & \textbf{Am. Photos} & \textbf{Co. CS} & \textbf{Co. Physics}  & \textbf{PPI} (F-Score) \\
\midrule

$\mathbf{D}^{-1}\mathbf{A}$ & 1 & 90.52 $\pm$ 0.09 & 92.81 $\pm$ 0.14 & 93.39 $\pm$ 0.04 & 95.89 $\pm$ 0.07 & 71.29 $\pm$ 0.11 \\
$\mathbf{D}^{-1}\mathbf{A}$ & 2 & 90.43 $\pm$ 0.16 & 92.84 $\pm$ 0.15 & 93.37 $\pm$ 0.04 & 95.89 $\pm$ 0.07 & 71.28 $\pm$ 0.08 \\
$\mathbf{D}^{-1}\mathbf{A}$ & 5 & 90.41 $\pm$ 0.17 & 92.84 $\pm$ 0.08 & 93.26 $\pm$ 0.06 & 95.86 $\pm$ 0.08 & 70.06 $\pm$ 0.08 \\
\midrule
$\mathbf{D}^{-\frac{1}{2}} {\mathbf{A}} {\mathbf{D}}^{-\frac{1}{2}}$ & 1 & 90.46 $\pm$ 0.06 & 92.89 $\pm$ 0.15 & 93.37 $\pm$ 0.02 & 95.92 $\pm$ 0.04 & 72.25 $\pm$ 0.13 \\
$\mathbf{D}^{-\frac{1}{2}} {\mathbf{A}} {\mathbf{D}}^{-\frac{1}{2}}$ & 2 & 90.25 $\pm$ 0.10  & 92.69 $\pm$ 0.14 & 93.35 $\pm$ 0.02 & 95.91 $\pm$ 0.04 &  71.94 $\pm$ 0.08\\
$\mathbf{D}^{-\frac{1}{2}} {\mathbf{A}} {\mathbf{D}}^{-\frac{1}{2}}$ & 5 & 90.08 $\pm$ 0.16 & 92.75 $\pm$ 0.11 & 93.23 $\pm$ 0.05 & 95.85 $\pm$ 0.05 & 70.08 $\pm$ 0.17 \\
\midrule
$\mathbf{I} - \mathbf{D}^{-\frac{1}{2}} {\mathbf{A}} {\mathbf{D}}^{-\frac{1}{2}}$ & 1 & 84.22 $\pm$ 0.28 & 90.23 $\pm$ 0.20 & 93.01 $\pm$ 0.12 &  95.50 $\pm$ 0.11 & 61.81 $\pm$ 0.08 \\
$\mathbf{I} - \mathbf{D}^{-\frac{1}{2}} {\mathbf{A}} {\mathbf{D}}^{-\frac{1}{2}}$ & 2 & 84.58 $\pm$ 0.22 & 90.75 $\pm$ 0.25 & 92.89 $\pm$ 0.14 & 95.46 $\pm$ 0.08  & 61.35 $\pm$ 0.13\\
$\mathbf{I} - \mathbf{D}^{-\frac{1}{2}} {\mathbf{A}} {\mathbf{D}}^{-\frac{1}{2}}$ & 5 & 88.65 $\pm$ 0.16 & 92.62 $\pm$ 0.20 & 91.82 $\pm$ 0.27 & 95.53 $\pm$ 0.09 & 64.50 $\pm$ 0.12 \\

\bottomrule
\end{tabular}
\end{adjustbox}

\vskip -0.1in
\end{table}

%% file: sections/07_discussion.tex
\section{Discussion}
\label{sec:discussion}
We have presented an algorithm that achieves state-of-the-art performance even though it is much simpler in terms of theoretical interpretation, architecture and training strategy than other methods. In this section, we detail aspects of RGI and its advantages and limitations with respect to other methods.

\paragraph{Augmentation-free and non-contrastive} Graph contrastive learning is the most popular approch to avoid the collapse of the representations by relying on negative pair sampling \cite{velickovic2018deep, sun2019infograph, Zhu:2020vf, NEURIPS2021_ff1e68e7, https://doi.org/10.48550/arxiv.2204.04874}. BGRL \cite{thakoor2021bootstrapped},  CCA-SSG \cite{zhang2021canonical} and G-BT \cite{bielak2021graph}, instead, do not need negative samples. The former avoids collapse with an asymmetric architecture and the other two, regularizing the covariance matrix of the representations. In this work, we also adopt a regularized solution since it is more interpretable and naturally avoids the collapse whereas it is still an open problem to theoretically demonstrate that bootstrapped methods avoid trivial solutions. However, these methods train an encoder with the invariance via data augmentation principle. It has been stated that transformations that drop information may modify the semantics of the graph so the invariance assumption may be incorrect and not hold for all graph domains. RGI, instead, requires no transformations as simply employs feature propagation to create the supervision signals, so that it is much more intuitive and only involves graph convolution-like operations.

\paragraph{Single branch architecture} Current state of the art algorithms usually rely on a joint embedding architecture that require multiple forward passes at each training step. For example, BGRL \cite{thakoor2021bootstrapped} forwards the graph through the encoder four times at each iteration. On the contrary, RGI is much more efficient and only performs one single forward pass while achieving similar performance to BGLR and other methods.

\paragraph{Simplicity and effectiveness} Being augmentation free, requiring a single branch architecture and being non-contrastive, substantially reduce both the time and space complexity of the algorithm for self-supervised learning on graphs with respect to other methods. First of all, the complexity of contrastive methods is quadratic with respect to the number of nodes, that is $\mathcal{O}(N^2)$. Instead, covariance regularization makes the complexity be quadratic on the embedding space dimensionality $\mathcal{O}(D^2)$, which is generally much lower than the number of nodes $D << N$. Additionally, being a single branch architecture reduces the time and space complexity of computing multiple forward and backward passes through the encoder during training.

\paragraph{Evaluation protocol}
In the context of self-supervised learning of graph neural networks, while the linear evaluation protocol is commonly used to compare different methods, it is important to consider the impact of other factors such as the differences in the encoder architecture, as well as the training protocol and evaluation-related hyperparameters. In our experiments, we observed a drop in performance on GraphMAE \cite{hou2022graphmae}, which reports state-of-the-art performance in the original paper, when we re-ran the method with the same dimensionality and linear evaluation training than the majority of the other methods of the comparison. This highlights the importance of a standardized evaluation protocol for self-supervised GNNs, which can help to ensure fair comparisons between different methods and provide a more accurate assessment of their performance. Therefore, we suggest that future work in the field of self-supervised learning on graphs should consider addressing these concerns.

\paragraph{Limitations}
RGI has been shown to be an effective algorithm to learn node-level representations in a self-supervised manner. However, by definition, RGI operates on homogeneous graphs so that, in general, it could not be directly applied to heterogeneous scenarios. Although RGI requires a single branch architecture, which simplifies the training procedure with respect other methods, it comes at expense of extra depth to obtain the propagated embeddings, which increase the complexity in settings in which neighbor sampling is a must. Nonetheless, our experiments show that $K=1$ is generally the best choice as the number of steps to propagate the information through the graph, so that this issue is alleviated.

%% file: sections/08_conclusions.tex
\section{Conclusion}
\label{sec:conclusion}
In this work we have introduced RGI, a simple yet effective framework for self-supervised learning on graphs based on the propagation of node embeddings to generate supervision signals. The objective function maximizes the mutual information between node embeddings and their propagation by addressing the reconstruction between them and regularizing the covariance matrix of the representations to avoid the total and dimensional collapse. We have shown that, in spite of its simplicity and that it does not require training for too long, it achieves state-of-the-art performance on multiple transductive and inductive datasets following the linear evaluation protocol on downstream tasks.

%% file: sections/appendix.tex
\section{Implementation details}
\label{sec:implementation}

\begin{algorithm}[tb]
   \caption{Pseudo-code PyTorch implementation of RGI}
   \label{algo:pytorch}
\begin{algorithmic}
    \STATE \
    \STATE \PyComment{obtain local views}
    \STATE \PyCode{u = f(x, adj)}
    \STATE \PyComment{get symmetric adjacency matrix}
    \STATE \PyCode{deg = adj.sum(dim=-1)} \PyComment{adj dense and symmetric}
    \STATE \PyCode{deg = diagonal(inverse(deg).sqrt())}
    \STATE \PyCode{shift = deg @ adj @ deg}
    \STATE \PyComment{obtain global views}
    \STATE \PyCode{v = matrix\_power(shift,K) @ u}
    \STATE \PyComment{batch size and dimensionality}
    \STATE \PyCode{N, D = u.size()}
    \STATE \PyComment{reconstruction between views}
    \STATE \PyCode{v\_prime = h\_phi(u)}
    \STATE \PyCode{u\_prime = h\_psi(v)}
    \STATE \PyComment{mean - centering}
    \STATE \PyCode{u\_norm = u - u.mean(dim=0)}
    \STATE \PyCode{v\_norm = v - v.mean(dim=0)}
    \STATE \PyComment{covariance matrices}
    \STATE \PyCode{cov\_u = (u\_norm.T @ u\_norm) / (N-1) }
    \STATE \PyCode{cov\_v = (v\_norm.T @ v\_norm) / (N-1) }
    \STATE \PyComment{reconstruction loss}
    \STATE \PyCode{rec\_loss = mse\_loss(u, u\_prime) + mse\_loss(v, v\_prime)}
    \STATE \PyComment{variance loss}
    \STATE \PyCode{var\_loss = diagonal(cov\_u).add\_(-1).pow\_(2).sum() / D + \textbackslash}
    \STATE \qquad \qquad \qquad \PyCode{diagonal(cov\_v).add\_(-1).pow\_(2).sum() / D}
    \STATE \PyComment{covariance loss}
    \STATE \PyCode{cov\_loss = off\_diagonal(cov\_u).pow\_(2).sum() / D + + \textbackslash}
    \STATE \qquad \qquad \quad \PyCode{off\_diagonal(cov\_v).pow\_(2).sum() / D }
    \STATE \PyComment{total loss}
    \STATE \PyCode{loss = lambd\_1 * rec\_loss + lambd\_2 * var\_loss + lambd\_3 * cov\_loss}
    \STATE \PyComment{optimizer}
    \STATE \PyCode{loss.backward()}
    \STATE \PyCode{optimizer.step()}
\end{algorithmic}
\end{algorithm}

\paragraph{Implementation} RGI and all neural networks are implemented in PyTorch \cite{paszke2017automatic} and PyTorch Geometric \cite{Fey/Lenssen/2019}. Algorithm \ref{algo:pytorch} shows a PyTorch-like pseudo-code implementation of RGI. All our experiments are run in a single 16GB GPU. For linear evaluation, we have employed Sci-kit Learn library \cite{scikit-learn} for \textit{Amazon Computers}, \textit{Amazon Photos}, \textit{Coauthor CS} and \textit{Coauthor Physics} datasets and PyTorch for the \textit{PPI} dataset, taking the implementation from \cite{thakoor2021bootstrapped}.

\paragraph{Optimization} All models have been trained with Adam optimizer \cite{DBLP:journals/corr/KingmaB14} and a weight decay of $10^{-5}$. Additionally, we have employed a learning rate scheduler with linear warmup for $n_{warmup}$ epochs and cosine decay for the remaining $n_{epochs} - n_{warmup}$, where $n_{epochs}$ is the total number of epochs. The values of $n_{warmup}$ and $n_{epochs}$ for the different datasets can be found in Table \ref{tab:hyperparameters}. In transductive settings, we perform full-graph optimization at each gradient step. Alternatively, for the \textit{PPI} dataset, which is multi-graph, we set a batch size of 1 graph for each gradient step due to memory constraints of the GAT encoder. 

\paragraph{Node feature normalization} Input node features are normalized row-wise with the L1 norm for transductive settings whereas no normalization is applied to \textit{PPI} node features. Before fitting the linear classifier with the obtained representations, they are row-wise L2 normalized.

\paragraph{Hyperparameters} Some of the hyperparameters have been tuned with a small search whereas others have been fixed. Table \ref{tab:hyperparameters} shows the hyper-parameters settings employed to obtain the results in Table \ref{tab:results_transductive} and Table \ref{tab:results_ppi}. The search space of the optimized hyper-parameters is the following:

\begin{itemize}
    \item Propagation matrix : $\{ \mathbf{D}^{-1} \mathbf{A}, \mathbf{D}^{-\frac{1}{2}} {\mathbf{A}} {\mathbf{D}}^{-\frac{1}{2}}, \mathbf{I} - \mathbf{D}^{-\frac{1}{2}} {\mathbf{A}} {\mathbf{D}}^{-\frac{1}{2}}\}$
    \item Number of global propagation steps $(K)$ : $\{ 1, 2, 5\}$
    \item Number of training epochs $(n_{epochs})$ : $\{ 1000, 2000, 5000 \}$
    \item Learning rate $(lr)$ : $\{ 10^{-3}, 10^{-4}, 10^{-5} \}$
    \item Reconstruction loss weight $(\lambda_1)$: $\{10, 15, 20\}$
    \item Variance loss weight $(\lambda_2)$: $\{5, 10, 15\}$
    \item Covariance loss weight $(\lambda_3)$: $\{1, 5, 10\}$
    \item Dropout regularization probability input graph $(p_{input})$ : $\{ 0.0, 0.5\}$
    \item Dropout regularization probability before propagation $(p_{local})$ : $\{ 0.0, 0.5\}$  
\end{itemize}
The number of GNN layers $(L)$ and embedding size $(D)$ have been fixed according to previous reports for better comparability. The number of warmup epochs is set to be $n_{epochs} / 10$.

\begin{table}[ht]
\centering
\caption{Hyperparameters.}
\label{tab:hyperparameters}
\begin{adjustbox}{width=1\textwidth}
\small
\begin{tabular}{lccccc}
\toprule

\textbf{Method} & \textbf{Am. Computers} & \textbf{Am. Photos} & \textbf{Co. CS} & \textbf{Co. Physics} & \textbf{PPI} \\
\midrule
$L$ & 2 & 2 & 2 & 2 & 3\\
$K$ & 1 & 1 & 1 & 1 & 1\\
$D$ & 512 & 512 & 512 & 512 & 512 \\
Propagation & $\mathbf{D}^{-\frac{1}{2}} {\mathbf{A}} {\mathbf{D}}^{-\frac{1}{2}}$ & $\mathbf{D}^{-\frac{1}{2}} {\mathbf{A}} {\mathbf{D}}^{-\frac{1}{2}}$ & $\mathbf{D}^{-\frac{1}{2}} {\mathbf{A}} {\mathbf{D}}^{-\frac{1}{2}}$ & $\mathbf{D}^{-\frac{1}{2}} {\mathbf{A}} {\mathbf{D}}^{-\frac{1}{2}}$ & $\mathbf{D}^{-\frac{1}{2}} {\mathbf{A}} {\mathbf{D}}^{-\frac{1}{2}}$ \\
\midrule
$n_{epochs}$ & 5000 & 1000 & 1000 & 1000 & 2000\\
$n_{warmup}$ & 500 & 100 & 100 & 100 & 200 \\
$lr$ & $10^{-4}$ & $10^{-4}$ & $10^{-5}$ & $10^{-5}$ & $10^{-4}$ \\
\midrule
$\lambda_1$ & 10 & 10 & 20 & 20 & 15 \\
$\lambda_2$ & 5 & 5 & 15 & 15 & 10\\
$\lambda_3$ & 1 & 1 & 1 & 1 & 10\\
\midrule
$p_{input}$ & 0.5 & 0.5 & 0.5 & 0.5 & 0.0 \\
$p_{local}$ & 0.0 & 0.0 & 0.0 & 0.0 & 0.5 \\
\bottomrule
\end{tabular}

\end{adjustbox}
\end{table}

\section{Dataset details}
\label{sec:datasets}

\textbf{Amazon Computers, Amazon Photos} are extracted from the Amazon co-purchase graph \cite{10.1145/2766462.2767755} whose nodes represent products and edges encode whether two elements are usually purchased together. Node features are a vector representation of a bag-of-words from product's reviews and nodes are classified into 10 (for Computers) and 8 (for Photos) classes, given by the product category. Since there is no standard split for these datasets, we randomly split the nodes into (10/10/80\%) for train, validation and test, respectively.

\textbf{Coauthor Computer Science, Coauthor Physics} are from the Microsoft Academic Graph \cite{10.1145/2740908.2742839} from the KDD Cup 2016 challenge. Nodes represent authors, which are connected by an edge if they have co-authored a paper. Node features encode the keywords of each author's papers and authors are labeled into 15 (for CS) and 5 (for Physics) according to their most active field of study. We also use a random split into (10/10/80\%) for train, validation and test, respectively.

\textbf{PPI} is a proteint-protein interaction network \cite{Zitnik2017}. It consists of 24 independent graphs that correspond to different human tissues, whose nodes are proteins and edges represent interactions between them. Node features are biological properties and they are labeled according to the protein functions. We employ 20 graphs for training, 2 for validation and 2 for testing.